\documentclass[letterpaper, 10 pt, conference]{ieeeconf}  

\usepackage[hidelinks]{hyperref}
\hypersetup{
    colorlinks=true,
    linkcolor=black,
    filecolor=magenta,      
    urlcolor=blue,
}
\usepackage[noadjust]{cite}
\usepackage{amsmath}
\usepackage{graphicx}
\usepackage{amssymb}
\usepackage{colortbl}
\usepackage{arydshln}
\usepackage{stfloats}
\usepackage{lipsum}
\usepackage{algpseudocode}
\usepackage{multirow}
\usepackage{extarrows}
\usepackage{mathrsfs}
\usepackage{upgreek}

\newtheorem{remark}{Remark}

\newtheorem{example}{Example}
\newtheorem{assumption}{Assumption}
\newtheorem{proposition}{Proposition}
\newtheorem{definition}{Definition}

\definecolor{mygray}{gray}{0.8}
\usepackage[table]{xcolor}

\IEEEoverridecommandlockouts                              
\title{\LARGE \bf Approximate Bisimulation Relations for Neural Networks and Application to Assured Neural Network Compression}

\author{
Weiming Xiang,~\IEEEmembership{Senior Member, IEEE} and Zhongzhu Shao
\thanks{W. Xiang is with the School of Computer and Cyber Sciences, Augusta University, Augusta GA 30912 USA. Email: {\tt\small wxiang@augusta.edu}}      
\thanks{Z. Shao is with Department of Electrical Engineering, Southwest Jiaotong University, Chengdu, China.}
}

\begin{document}

\maketitle
\thispagestyle{empty}
\pagestyle{empty}

\begin{abstract}
\boldmath
In this paper, we propose a concept of approximate bisimulation relation for feedforward neural networks. In the framework of approximate bisimulation relation, a novel neural network merging method is developed to compute the approximate bisimulation error between two neural networks based on reachability analysis of neural networks. The developed method is able to quantitatively measure the distance between the outputs of two neural networks with same inputs. Then, we apply the approximate bisimulation relation results to perform neural networks model reduction and compute the compression precision, i.e.,  assured neural networks compression. At last, using the assured neural network compression, we accelerate the verification processes of ACAS Xu neural networks to illustrate the effectiveness and advantages of our proposed  approximate bisimulation approach.
\end{abstract}

\section{Introduction}

Deep neural networks (DNN) are now widely used in a variety of contemporary applications, such as image processing \cite{Litjens2017}, pattern recognition \cite{Schmidhuber2015,Lawrence1997}, adaptive control, \cite{Hunt1992,Tong2016} autonomous vehicles \cite{bojarski2016end}, and other fields, showing the powerful capabilities solving complex and challenging problems that traditional approaches fail to deal with. To cope with complex tasks and different environments, neural network models have been being developed with increasing scale and complexity, which aim to provide better performance and higher accuracy. However, the increasing scale and complexity of the neural network models also mean that neural networks require a large amount of resources for real-world implementation such as higher memory, more computational power, and higher energy consumption \cite{Simon2020}. Therefore, neural network model compression methods were developed to reduce the complexity of neural networks at the least possible price of performance deterioration. For instance, in \cite{Yabo2019}, four compression methods for deep convolutional neural networks are summarized, but some problems were pointed out such as a shxarp drop in the accuracy of the network when compressing. More neural network compression results can be found in the recent survey \cite{deng2020model} and references therein. Moreover, it has been observed that well-trained neural networks on abundant data are sometimes sensitive to updates, and react in unexpected and incorrect ways to even slight changes of the parameters \cite{szegedy2013intriguing}. The neural network compression inevitably introduces changes to the neural network. Therefore, an approach is needed to formally characterize the changes between the original neural network model and its compressed version. 

In this paper, we propose an approximate bisimulation relation between two neural networks, which formally characterize the maximal difference between the outputs of two neural networks generated from same inputs. Based on the framework of the approximate bisimulation relation, we propose a neural network merging algorithm to calculate the approximate bisimulation error, measuring the distance between two neural networks. Applying this approximate bisimulation method to neural network model compression, we can obtain the precision of neural network model compression, which is able to provide assurance to perform tasks using compressed neural networks on behalf of original ones. To illustrate the feasibility of the approximate bisimulation method, we apply it to accelerate verification processes of the ACAS Xu neural networks using the compressed neural networks.

The remainder of the paper is organized as follows: Preliminaries are given in Section II. The approximate bisimulation relation and approximate bisimulation error computation are presented in Section III. Assured neural network compression and  examples are given in Section IV. The conclusion is presented in Section V.

\textit{Notations:} For the rest of paper, $\mathbf{0}_{n \times m}$ denotes denotes a matrix of $n$ rows and $m$ columns with all elements zero, $\mathbf{I}_n$ denotes the $n$-dimensional unit matrix. $\mathsf{purelin}(\cdot)$ is linear transfer function, i.e., $x=\mathsf{purelin}(x)$.

\section{Preliminaries}
In this paper, we consider a class of feedforward neural networks which generally consist of one input layer, multiple hidden layers and one output layer. Each layer consists of one or multiple neurons. The action of a neuron depends on its activation function, which is in the description of
\begin{align} \label{neuron}
    y_i=\phi{(\sum\nolimits_{j=1}^{n} w_{ij}x_{j}+b_i)}
\end{align}
where $y_i$ is the output of the $i$th neuron, $x_j$ is the $j$th input of the $i$th neuron, $w_{ij}$ is the weight from the $j$th input to the $i$th neuron, $b_i$ is the bias of the $i$th input, $\phi(\cdot)$ is the activation function. 
Each layer $\ell$ ($1 \le \ell \le L$) of a feedforward neural network has $n^{\{\ell\}}$ neurons. Layer $\ell = 0$ denote the input layer, $n^{\{0\}}$ denote the number of the input layer. For the layer $\ell$, the input vector is denoted by $\mathbf{x}^{\{\ell\}}$, respectively, the weight matrix and the bias vector are 
\begin{align} 
    \mathbf{W}^{\{\ell\}}&=[w_1^{\{\ell\}},\cdots,w_{n^{\{\ell\}}}^{\{\ell\}}]^\mathrm{T} \\
    \mathbf{b}^{\{\ell\}}&=[b_1^{\{\ell\}},\cdots,b_{n^{\{\ell\}}}^{\{\ell\}}]^\mathrm{T}
\end{align}
where $w_i^{\{\ell\}}$ is the weight vector, $b_i^{\{\ell\}}$ is the bias value. The output vector of layer $\ell$ is $\mathbf{y}^{\{\ell\}}$ defined by
\begin{align}
    \mathbf{y}^{\{\ell\}}=\phi^{\{\ell\}}(\mathbf{W}^{\{\ell\}}\mathbf{x}^{\{\ell\}}+\mathbf{b}^{\{\ell\}})
\end{align}
where $\phi^{\{\ell\}}(\cdot)$ is the activation function of layer $\ell$.

For the whole neural network, the input and output layer are $\mathbf{x}^{[0]}$ and $\mathbf{y}^{[L]}$ respectively, the input of the layer $\ell$ is the output of the layer $\ell-1$, the mapping relation from the input to the output is denoted by
\begin{align}\label{NN}
   \mathbf{y}^{\{L\}}=\Phi(\mathbf{x}^{\{0\}})
\end{align}
where $\Phi(\cdot) \triangleq \phi^{\{L\}} \circ \phi^{\{L-1\}} \cdots \phi^{\{1\}}(\cdot)$. The mapping relation $\Phi$ includes not only the activation function of the neural network, but also the weight matrix and the bias vectors, which represent the structural information of the neural network.

Given an input set $\mathcal{X}$, the output reachable set of a neural network is stated by the definition below. 
\begin{definition}\label{def1}
Given a neural network in the form of (\ref{NN}) and input set $\mathcal{X} \in \mathbb{R}^{n^{\{0\}}}$, the following set
\begin{align}\label{reachset}
   \mathcal{Y}= \left\{ \mathbf{y}^{\{L\}} \in \mathbb{R}^{n^{\{L\}}} \mid \mathbf{y}^{\{L\}} = \Phi(\mathbf{x}^{\{0\}}), \mathbf{x}^{\{0\}} \in \mathcal{X} \right\}
\end{align}
is called the output reachable set of neural network {(\ref{NN})}.
\end{definition}

The safety specification of a neural network is expressed by the set defined in the output space, describing the safety requirement.

\begin{definition} \label{def2}
Safety specification $\mathcal{S}$ formalizes the safety requirement for output $\mathbf{y}^{\{L\}}$ of neural network {(\ref{NN})}, and is a predicate over output $\mathbf{y}^{\{L\}}$ of neural network {(\ref{NN})}. The neural network {(\ref{NN})} is safe if and only if the following condition is satisfied:
\begin{align}
   \mathcal{Y} \cap \neg \mathcal{S} = \emptyset 
\end{align}
where $\mathcal{Y}$ is the output set defined by (\ref{reachset}), and $\neg$ is the symbol for logical negation.
\end{definition}

The above safety verification concept is reachability-based and will be used in Section IV for safety verification of neural networks of Airborne  Collision  Avoidance Systems in \cite{owen2019acas}.

\section{Approximation Simulation Relations of Neural Networks}

\subsection{Approximation Bisimulation Relations}


In order to characterize the difference of two feedforward neural networks in terms of outputs, we defined the following metric which measures the distance between the outputs of two neural networks in the framework of the reachable set defined in Definition \ref{def1}.

\begin{definition}\label{def2} Consider two neural networks $\mathbf{y}^{\{L\}}=\Phi_j(\mathbf{x}^{\{0\}})$, $j \in \{1,2\}$, input set $\mathcal{X} \in \mathbb{R}^{n^{\{0\}}}$, and output sets $\mathcal{Y}_j \in \mathbb{R}^{n^{\{L\}}}$, $j \in \{1,2\}$, we define $\mathcal{N}_j = (\mathcal{X},\mathcal{Y}_j,{\Phi}_j)$, $j \in \{1,2\}$, and 
\begin{align} \label{def2_1}
   d(\Phi_1(\mathbf{x}^{\{0\}}_1),\Phi_2(\mathbf{x}^{\{0\}}_2)) = 
   \begin{cases}
   \rho(\mathbf{y}^{\{L\}}_1,\mathbf{y}^{\{L\}}_2) &\mbox{ if } \mathbf{x}^{\{0\}}_1 = \mathbf{x}^{\{0\}}_2\\
   +\infty &\mbox{ otherwise }
   \end{cases}
\end{align}
where 
\begin{align}\label{rho}
    \rho(\mathbf{y}^{\{L\}}_1,\mathbf{y}^{\{L\}}_2) = \sup \limits_{\mathbf{y}^{\{L\}}_1 \in \mathcal{Y}_1, \mathbf{y}^{\{L\}}_2 \in \mathcal{Y}_2} \left\|{\mathbf{y}^{\{L\}}_1}-{\mathbf{y}^{\{L\}}_2}\right\| .
\end{align}
\end{definition}

It is noted that  $d(\Phi_1(\mathbf{x}^{\{0\}}_1),\Phi_2(\mathbf{x}^{\{0\}}_2))$ defined in (\ref{def2_1}) characterizes the maximal difference between the outputs of two neural networks generated from a same input, which quantifies the discrepancy between two neural networks ${\Phi}_1$ and ${\Phi}_2$ in terms of outputs. Based on Definition \ref{def2}, we will be able to establish the approximate bisimulation relation of two neural networks.  

\begin{definition}
Consider $\mathcal{N}_j = (\mathcal{X},\mathcal{Y}_j,{\Phi}_j)$, $j \in \{1,2\}$, and let $\varepsilon \ge 0$, a relation $\mathscr{R}_\varepsilon \in \mathbb{R}^{n^{\{L\}}} \times \mathbb{R}^{n^{\{L\}}}$ is called an approximate simulation relation between $\mathcal{N}_1$ and $\mathcal{N}_2$, of precision $\varepsilon$, if for all $({\mathbf{y}^{\{L\}}_1},\mathbf{y}^{\{L\}}_2) \in \mathcal{R}_\varepsilon$
\begin{enumerate}
    \item $d(\Phi_1(\mathbf{x}^{\{0\}}),\Phi_2(\mathbf{x}^{\{0\}})) \le \varepsilon$, $\forall \mathbf{x}^{\{0\}} \in \mathcal{X}$;
     \item $\forall \mathbf{x}^{\{0\}} \in \mathcal{X}$, $\forall \Phi_1(\mathbf{x}^{\{0\}}) \in \mathcal{Y}_1$, $\exists  \Phi_2(\mathbf{x}^{\{0\}})  \in \mathcal{Y}_2$ such that $(\Phi_1(\mathbf{x}^{\{0\}}), \Phi_2(\mathbf{x}^{\{0\}})) \in \mathscr{R}_\varepsilon$;
    \item $\forall \mathbf{x}^{\{0\}} \in \mathcal{X}$, $\forall \Phi_2(\mathbf{x}^{\{0\}})\in {\mathcal{Y}}_2$, $\exists  \Phi_1(\mathbf{x}^{\{0\}}) \in \mathcal{Y}_1$ such that $(\Phi_1(\mathbf{x}^{\{0\}}), \Phi_2(\mathbf{x}^{\{0\}})) \in \mathscr{R}_\varepsilon$
\end{enumerate}
and we say neural networks $\mathcal{N}_1$ and $\mathcal{N}_2$  are approximately bisimilar with precision $\varepsilon$, denoted by $\mathcal{N}_1 \sim_{\varepsilon} \mathcal{N}_2$ .
\end{definition}
\begin{remark}
The meaning of approximate bisimulation between two neural networks $\mathcal{N}_1$ and $\mathcal{N}_2$ with precision $\varepsilon$, which denoted by $\mathcal{N}_1 \sim_{\varepsilon} \mathcal{N}_2$, is as follows: 
Considering two neural networks $\mathcal{N}_1$ and $\mathcal{N}_2$ and any output of neural network $\mathcal{N}_1$, we can find one output generated by the same corresponding input out of neural network $\mathcal{N}_2$, and vice versa. The two outputs of two neural networks always satisfy that the distance between them is bounded by $\varepsilon$. In the case of $\varepsilon = 0$, we can define that the two neural networks have an exact simulation relation. 
\end{remark}

Then, we define metrics measuring the distance between the observed behaviors of neural network  $\mathcal{N}_1$ and $\mathcal{N}_2$. Based on the defined notion of approximate bisimulation, we can define the approximate bisimulation error to represent the distance between two neural networks. 

\begin{definition}
Given two neural networks $\mathcal{N}_1$ and $\mathcal{N}_2$, the approximate bisimulation error of them are defined by 
\begin{align}\label{metric}
   d(\mathcal{N}_1,\mathcal{N}_2) &= \sup\{\varepsilon \mid \mathcal{N}_1 \sim_\varepsilon \mathcal{N}_2\}
\end{align}
where $\varepsilon \ge 0$.
\end{definition}

The key to establish the approximation bisimulation relation between two neural networks is how to efficiently compute the approximation bisimulation error defined by (\ref{metric}). In the next subsection, a reachability-based method is proposed to compute the approximate bisimulation error.

\subsection{Approximate Bisimulation Error Computation}
In order to compute the approximate bisimulation error $\varepsilon$ between two neural network outputs, the set-valued reachability methods can be used. First, consider two neural networks with same input set $\mathcal{X}$, a feedforward neural network $\mathcal{N}_L$ with $L$ hidden layers and $n^{\{l\}}$, $l = 1,\ldots,L$ neurons in each layer, and its bisimilar feedforward neural network  $\mathcal{N}_S$ with $S$ hidden layers and $n^{\{s\}}$, $l = 1,\ldots,S$ neurons in each hidden layer. 

Without loss of generality, the following assumption is given for neural networks $\mathcal{N}_L$ and $\mathcal{N}_S$. 

\begin{assumption}\label{assumption_1}The following assumptions hold for two neural networks $\mathcal{N}_L$  and $\mathcal{N}_s$:
\begin{enumerate}
    \item The number of inputs of two neural networks are same, i.e., $n_L^{\{0\}} = n_S^{\{0\}}$;
    \item The number of outputs of two neural networks are same, i.e., $n_L^{\{L\}} = n_S^{\{S\}}$;
    \item The number of hidden layers of neural network $\mathcal{N}_L$ is greater than or equal the number of hidden layers of neural network $\mathcal{N}_S$, i.e., $L \ge S$.
\end{enumerate}

\end{assumption}

According to (\ref{rho}), (\ref{metric}), the approximate bisimulation error between $\mathcal{N}_L$ and $\mathcal{N}_S$ can be expressed by 
\begin{align} \label{distance}
    d(\mathcal{N}_L,\mathcal{N}_S)= \sup \limits_{\mathbf{x}^{\{0\}} \in \mathcal{X}} \left\|\Phi_L(\mathbf{x}^{\{0\}})-\Phi_S(\mathbf{x}^{\{0\}})\right\| .
\end{align}
To obtain the approximate bisimulation error of the two neural networks, i.e., $d(\mathcal{N}_L,\mathcal{N}_S)$, we propose to merge the two neural networks in a non-fully connected structure $\mathcal{N}_M$, which is able to generate the output $\mathbf{y}^{\{M\}}_M$ exactly characterizing the difference of the outputs of $\mathcal{N}_L$ and $\mathcal{N}_S$, i.e., $\mathbf{y}^{\{M\}}_M = \mathbf{y}^{\{L\}}_L - \mathbf{y}^{\{S\}}_S$. 

\noindent \textbf{Merged Neural Network $\mathcal{N}_M$: } 
To begin with, we consider two neural networks $\mathcal{N}_L$ and $\mathcal{N}_S$ with same input $\mathbf{x}^{\{0\}}$. We use $\mathbf{W}_M^{\{m\}}$ and $\mathbf{b}_M^{\{m\}}$  to denote the weight matrix and bias vector of the $m$th layer of the merged neural network $\mathcal{N}_M$, $\mathbf{x}^{\{m\}}_M$ and $\mathbf{y}^{\{m\}}_M$ are input and output vectors of $m$th layer of $\mathcal{N}_M$. The structure of the merged neural network $\mathcal{N}_M$ with $L+1$ layers is recursively defined as below: 
\begin{align}\label{mNN}
\begin{cases}
    \mathbf{y}_M^{\{m\}} = \phi_{M}^{\{m\}}(\mathbf{W}^{\{m\}}_M\mathbf{x}_M^{\{m-1\}}+\mathbf{b}^{\{m\}}_M)
    \\
    \mathbf{x}_M^{\{m\}} = \mathbf{y}_M^{\{m\}}
    \end{cases}
\end{align}
where $m = 1,2,\ldots,L+1$. The input is $\mathbf{x}_M^{\{0\}} = \mathbf{x}^{\{0\}} $, output is $\mathbf{y}_M^{\{L+1\}}$,  weight matrices $\mathbf{W}^{\{m\}}_M$ and bias vectors $\mathbf{b}^{\{m\}}_M$, and activation functions $\phi_{M}^{\{m\}}(\cdot)$ are categorized as the following five cases: 

\begin{enumerate}
    \item When $m = 1$, $\mathbf{W}^{\{1\}}_M$,  $\mathbf{b}^{\{1\}}_M$, and  $\phi_{M}^{\{1\}}(\cdot)$ are
    \begin{align} \label{mNN_1_1}
    \mathbf{W}^{\{1\}}_M&= \begin{bmatrix}
    \mathbf{W}^{\{1\}}_L \\ \mathbf{W}^{\{1\}}_S
    \end{bmatrix}
    \\
    \label{mNN_1_2}
    \mathbf{b}^{\{1\}}_M &= \begin{bmatrix}
    \mathbf{b}^{\{1\}}_L \\ \mathbf{b}^{\{1\}}_S
    \end{bmatrix}
    \\\
    \label{mNN_1_3}
   \phi^{\{1\}}_M(\cdot) &= \begin{bmatrix}
      \phi^{\{1\}}_L(\cdot) \\   \phi^{\{1\}}_S(\cdot)
    \end{bmatrix} .
\end{align}
\item When $1 < m \le S-1$,
 $\mathbf{W}^{\{m\}}_M$,  $\mathbf{b}^{\{m\}}_M$, and $\phi_{M}^{\{m\}}(\cdot)$ are
\begin{align} \label{mNN_2_1}
    \mathbf{W}^{\{m\}}_M &= \begin{bmatrix}
    \mathbf{W}^{\{m\}}_L & \mathbf{0}_{{n^{\{m\}}_L} \times {n^{\{m-1\}}_S}}
    \\
    \mathbf{0}_{{n^{\{m\}}_S} \times {n^{\{m-1\}}_L}} & \mathbf{W}^{\{m\}}_S
    \end{bmatrix}
    \\
        \label{mNN_2_2}
    \mathbf{b}^{\{m\}}_M  &= \begin{bmatrix}
    \mathbf{b}^{\{m\}}_L
    \\
    \mathbf{b}^{\{m\}}_S
    \end{bmatrix}
    \\
     \label{mNN_2_3}
   \phi^{\{m\}}_M(\cdot) &= \begin{bmatrix}
      \phi^{\{m\}}_L(\cdot) \\   \phi^{\{m\}}_S(\cdot)
    \end{bmatrix} .
\end{align}
\item When $S-1 < m \le L-1$, $\mathbf{W}^{\{m\}}_M$,  $\mathbf{b}^{\{m\}}_M$, and $\phi_{M}^{\{m\}}(\cdot)$ are
\begin{align}
 \label{mNN_3_1}
    \mathbf{W}^{\{m\}}_M & = \begin{bmatrix}
    \mathbf{W}^{\{m\}}_L & \mathbf{0}_{{n^{\{m\}}_L} \times {n^{\{S-1\}}_S}}
    \\
    \mathbf{0}_{{n^{\{S-1\}}_S} \times {n^{\{m\}}_L}} & \mathbf{I}_{n^{\{S-1\}}_S}
    \end{bmatrix}
    \\
    \label{mNN_3_2}
    \mathbf{b}^{\{m\}}_M & = \begin{bmatrix}
    \mathbf{b}^{\{m\}}_L
    \\
    \mathbf{0}_{{n^{\{S-1\}}_S} \times 1}
    \end{bmatrix}
        \\
         \label{mNN_3_3}
   \phi^{\{m\}}_M(\cdot) &= \begin{bmatrix}
      \phi^{\{m\}}_L(\cdot) \\  
      \mathsf{purelin}(\cdot)
    \end{bmatrix} .
\end{align}
\item When $m = L$,  $\mathbf{W}^{\{L\}}_M$,  $\mathbf{b}^{\{L\}}_M$, and $\phi_{M}^{\{L\}}(\cdot)$ are
\begin{align}
   \label{mNN_4_1}
    \mathbf{W}^{\{L\}}_M &= \begin{bmatrix}
    \mathbf{W}^{\{L\}}_L & \mathbf{0}_{{n^{\{L\}}_L} \times {n^{\{S-1\}}_S}}
    \\
    \mathbf{0}_{{n^{\{S\}}_S} \times {n^{\{L-1\}}_L}} & \mathbf{W}^{\{S\}}_S 
    \end{bmatrix}
    \\
    \label{mNN_4_2}
    \mathbf{b}^{\{L\}}_M &= \begin{bmatrix}
    \mathbf{b}^{\{L\}}_L \\
    \mathbf{b}^{\{S\}}_S
    \end{bmatrix}
            \\
            \label{mNN_4_3}
   \phi^{\{L\}}_M(\cdot) &= \begin{bmatrix}
      \phi^{\{L\}}_L(\cdot) \\   \phi^{\{S\}}_S(\cdot)
    \end{bmatrix} .
\end{align}
\item When $m = L+1$,  $\mathbf{W}^{\{L+1\}}_M$,  $\mathbf{b}^{\{L+1\}}_M$, and $\phi_{M}^{\{L+1\}}(\cdot)$ are
\begin{align} \label{mNN_5_1}
    \mathbf{W}^{\{L+1\}}_M &= \begin{bmatrix}
    \mathbf{I}_{n^{\{L\}}_L} & -\mathbf{I}_{n^{\{L\}}_L} 
    \end{bmatrix}
    \\
    \label{mNN_5_2}
    \mathbf{b}^{\{L+1\}}_M &= \begin{bmatrix}
    \mathbf{0}_{{2n^{\{L\}}_L} \times 1} 
    \end{bmatrix}
    \\
     \label{mNN_5_3}
    \phi_{M}^{\{L+1\}}(\cdot) &= \mathsf{purelin}(\cdot) .
\end{align}
\end{enumerate}
\begin{figure*}[ht!]
\centering
	\includegraphics[width=\textwidth]{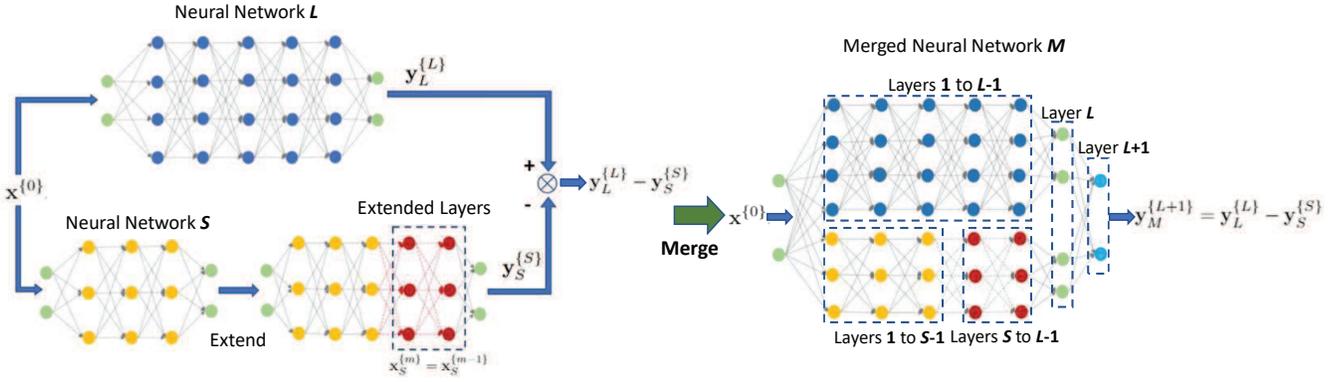}
	\caption{Neural network merging process for approximate bisimulation error computation}
	\label{nn_bisim} 
\end{figure*}
\begin{remark}
In the merging process of neural networks $\mathcal{N}_L$ and $\mathcal{N}_S$, (\ref{mNN_1_1})--(\ref{mNN_1_3}) ensures that  merged neural network $\mathcal{N}_M$ takes the one input $\mathbf{x}^{\{0\}}$ for the subsequent calls  involving both processes of $\mathcal{N}_L$ and $\mathcal{N}_S$. Then, for $1 < m \le S-1$, $\mathcal{N}_M$ conducts the computation of $\mathcal{N}_L$ and $\mathcal{N}_S$ parallelly for the hidden layers of $1 < m \le S-1$. When $S-1 < m \le L-1$,  the hidden layers of neural network $\mathcal{N}_S$ which has a less hidden layers are expanded to match the number of layers of  neural network $\mathcal{N}_L$ with larger number of hidden layers, but the expanded layers are forced to be pass the information to subsequent layers without any changes, i.e., the weight matrices of the expanded hidden layers are identity matrices, and the bias vector is the zero vectors. This expansion is formalized as (\ref{mNN_3_1})--(\ref{mNN_3_3}). Moreover, as $m=L$, this layer is a combination of output layers of both $\mathcal{N}_L$ and $\mathcal{N}_S$ to generate the same outputs of $\mathcal{N}_L$ and $\mathcal{N}_S$. At last, a comparison layer $L+1$ is added to compute the exact difference between two bisimular neural networks.
\end{remark}

With the merged neural network $\mathcal{N}_M$ in the description of (\ref{mNN})--(\ref{mNN_5_3}), we are ready to propose the main contribution of this work in Proposition \ref{proposition_1}.

\begin{proposition} \label{proposition_1}
Given two neural networks $\mathcal{N}_L$ with $L$ layers and $\mathcal{N}_S$ with $S$ layers under Assumption \ref{assumption_1}, the output $  \mathbf{y}_M^{\{L+1\}}$ of their merged neural network $\mathcal{N}_M$ defined by (\ref{mNN})--(\ref{mNN_5_3}) equals the difference of the output $\mathbf{y}_L^{\{L\}}$ of $\mathcal{N}_L$ and the output $\mathbf{y}_S^{\{S\}}$ of $\mathcal{N}_S$, i.e., 
\begin{align}
     \mathbf{y}_M^{\{L+1\}} = \mathbf{y}_L^{\{L\}} - \mathbf{y}_S^{\{S\}}
\end{align}
holds for any input $\mathbf{x}^{\{0\}} \in \mathcal{X}$, where $\mathbf{y}_L^{\{L\}}=\Phi_L(\mathbf{x}^{\{0\}})$ and $\mathbf{y}_S^{\{S\}}=\Phi_S(\mathbf{x}^{\{0\}})$.
\end{proposition}
\begin{proof}
 Considering an input $\mathbf{x}^{\{0\}} \in \mathcal{X}$ and according to (\ref{mNN_1_1})--(\ref{mNN_1_3}), the following results for output of layer $m=1$ of merged neural network $\mathcal{N}_M$ can be obtained
 \begin{align}
     \mathbf{y}_M^{\{1\}} 
     =
     \begin{bmatrix}
      \phi^{\{1\}}_L( \mathbf{W}^{\{1\}}_L\mathbf{x}^{\{0\}}+ \mathbf{b}^{\{1\}}_L) \\   \phi^{\{1\}}_S(\mathbf{W}^{\{1\}}_S\mathbf{x}^{\{0\}}+ \mathbf{b}^{\{1\}}_S)
     \end{bmatrix} =      \begin{bmatrix} \mathbf{x}_L^{\{1\}} 
       \\  \mathbf{x}_S^{\{1\}} .
     \end{bmatrix}
 \end{align}
 
Further considering layers $1 < m \le S-1$ of $\mathcal{N}_M$, and using (\ref{mNN_2_1}) and (\ref{mNN_2_2}), it leads to
 \begin{align}
     \mathbf{W}^{\{m\}}_M \mathbf{x}_M^{\{m-1\}} + \mathbf{b}^{\{m\}}_M &= \begin{bmatrix}
    \mathbf{W}^{\{m\}}_L\mathbf{x}_L^{\{m-1\}} + \mathbf{b}^{\{m\}}_L
    \\
     \mathbf{W}^{\{m\}}_S\mathbf{x}_S^{\{m-1\}}+\mathbf{b}^{\{m\}}_S
    \end{bmatrix}
 \end{align}
 where $1 < m \le S-1$. Then based on (\ref{mNN_2_3}), recursively we can obtain
 \begin{align}
     \mathbf{y}_M^{\{S-1\}} 
     =&
     \begin{bmatrix}
      \phi^{\{S-1\}}_L\circ \cdots \circ\phi^{\{1\}}_L( \mathbf{W}^{\{1\}}_L\mathbf{x}^{\{0\}}+ \mathbf{b}^{\{1\}}_L) \\  \phi^{\{S-1\}}_S\circ \cdots \circ\phi^{\{1\}}_S(\mathbf{W}^{\{1\}}_S\mathbf{x}^{\{0\}}+ \mathbf{b}^{\{1\}}_S)
     \end{bmatrix}
     \\
     = & \begin{bmatrix} \mathbf{x}_L^{\{S-1\}} 
       \\  \mathbf{x}_S^{\{S-1\}}
     \end{bmatrix} .
 \end{align}
 
 Moreover, considering $S-1 < m \le L-1$ and using (\ref{mNN_3_1}) and (\ref{mNN_3_2}), one can obtain
\begin{align}
    \mathbf{W}^{\{m\}}_M \mathbf{x}_M^{\{m-1\}} + \mathbf{b}^{\{m\}}_M &= \begin{bmatrix}
    \mathbf{W}^{\{m\}}_L\mathbf{x}_L^{\{m-1\}} + \mathbf{b}^{\{m\}}_L
    \\
     \mathbf{x}_S^{\{m-1\}}
    \end{bmatrix}
\end{align}
where $S-1 < m \le L-1$. From (\ref{mNN_3_3}), it yields
 \begin{align} \nonumber
     \mathbf{y}_M^{\{L-1\}} 
     =&
     \begin{bmatrix} \mathbf{x}_L^{\{L-1\}} 
       \\  \mathbf{x}_S^{\{S-1\}}
     \end{bmatrix}
 \end{align}
in which $\mathbf{x}_M^{\{L-1\}}$ is defined as
 \begin{align} \nonumber
     \mathbf{x}_M^{\{L-1\}}  &= \phi^{\{L-1\}}_L\circ \cdots \circ\phi^{\{S\}}_L( \mathbf{W}^{\{S\}}_L\mathbf{x}_L^{\{S-1\}}+ \mathbf{b}^{\{S\}}_L)
     \\
     &=\phi^{\{L-1\}}_L\circ \cdots \circ\phi^{\{1\}}_L( \mathbf{W}^{\{1\}}_L\mathbf{x}_L^{\{0\}}+ \mathbf{b}^{\{1\}}_L) .
 \end{align}
 
Then, as $m=L$ with (\ref{mNN_4_1}) and (\ref{mNN_4_2}) as well as $\mathbf{x}_M^{\{L-1\}} = \mathbf{y}_M^{\{L-1\}} $, it leads to
\begin{align}
     \mathbf{W}^{\{L\}}_M \mathbf{x}_M^{\{L-1\}} + \mathbf{b}^{\{L\}}_M &= \begin{bmatrix}
    \mathbf{W}^{\{L\}}_L\mathbf{x}_L^{\{L-1\}} + \mathbf{b}^{\{L\}}_L
    \\
     \mathbf{W}^{\{S\}}_S\mathbf{x}_S^{\{S-1\}}+\mathbf{b}^{\{S\}}_S
    \end{bmatrix}
 \end{align}

Also due to (\ref{mNN_4_3}), we can have
\begin{align}
     \mathbf{y}_M^{\{L\}} 
     =
     \begin{bmatrix}
    \phi_{L}^{\{L\}}(\mathbf{W}^{\{L\}}_L\mathbf{x}_L^{\{L-1\}} + \mathbf{b}^{\{L\}}_L)
    \\
     \phi_{S}^{\{S\}}(\mathbf{W}^{\{S\}}_S\mathbf{x}_S^{\{S-1\}}+\mathbf{b}^{\{S\}}_S)
    \end{bmatrix} =      \begin{bmatrix} \mathbf{y}_L^{\{L\}} 
       \\  \mathbf{y}_S^{\{S\}}
     \end{bmatrix} .
 \end{align}
 
 At last, when $m=L+1$ with (\ref{mNN_5_1})--(\ref{mNN_5_3}), the following result can be obtained
 \begin{align} \nonumber
     \mathbf{y}_M^{\{L+1\}} &= \mathbf{W}^{\{L+1\}}_M \mathbf{x}_L^{\{M\}}  + \mathbf{b}^{\{L+1\}}_M 
     \\\nonumber
    & = \begin{bmatrix}
    \mathbf{I}_{n^{\{L\}}_L} & -\mathbf{I}_{n^{\{L\}}_L} 
    \end{bmatrix} \begin{bmatrix} \mathbf{y}_L^{\{L\}} 
       \\  \mathbf{y}_S^{\{S\}}
     \end{bmatrix}
     \\
     &= \mathbf{y}_L^{\{L\}} -  \mathbf{y}_S^{\{S\}} .
 \end{align}
where $\mathbf{y}_L^{\{L\}}=\Phi_L(\mathbf{x}^{\{0\}})$ and $\mathbf{y}_S^{\{S\}}=\Phi_S(\mathbf{x}^{\{0\}})$.
The proof is complete.
\end{proof}

Proposition \ref{proposition_1} implies that, for any individual input $\mathbf{x}^{\{0\}}$, we can compute the difference of the outputs between two bisimilar neural networks via generating the output of their merged neural network of $\mathbf{x}^{\{0\}}$. This lays the foundation of computing the approximate bisimulation error in the description of (\ref{distance}), i.e., the computation of the maximum discrepancy between two bisimilar neural networks subject to an input set $\mathcal{X}$ can be converted to the output reachable set $\mathcal{Y}_M$ computation of merged neural network $\mathcal{N}_M$.


\begin{proposition}\label{proposition_2}
Given an input set $\mathcal{X}$, two neural networks $\mathcal{N}_L$ with $L$ layers and $\mathcal{N}_S$ with $S$ layers under Assumption \ref{assumption_1}, their merged neural network $\mathcal{N}_M$ can be defined by (\ref{mNN})--(\ref{mNN_5_3}). Then, the  approximate bisimulation error between $\mathcal{N}_L$ and $\mathcal{N}_S$ can be computed by 
\begin{align} \label{distance_reach}
    d(\mathcal{N}_L,\mathcal{N}_S)= \sup \limits_{\mathbf{y}_M^{\{L+1\}} \in \mathcal{Y}_M}\left\|\mathbf{y}_M^{\{L+1\}}\right\| 
\end{align}
where $\mathbf{y}_M^{\{L+1\}}=\Phi_M(\mathbf{x}^{\{0\}})$ is the output of $\mathcal{N}_M$ and $\mathcal{Y}_M$ is the output reachable set of $\mathcal{N}_M$.
\end{proposition}
\begin{proof} The result can be obtained straightforwardly from the result in  Proposition \ref{proposition_1}, i.e., $\mathbf{y}_M^{\{L+1\}} = \mathbf{y}_L^{\{L\}} - \mathbf{y}_S^{\{S\}}$. The proof is complete.
\end{proof}

As shown in Proposition \ref{proposition_2}, the key of computing $d(\mathcal{N}_L,\mathcal{N}_S)$ is to compute the output reachable set $\mathcal{Y}_M$. For instance, as in NNV neural network reachability analysis tool, the reachable sets are in the form of a family of polyhedral sets \cite{tran2020nnv}, and in IGNNV tool, the output reachable set is a family of interval sets \cite{xiang2018output,xiang2020reachable}. With the reachable set $\mathcal{Y}_M$, the approximate bisimulation error $d(\mathcal{N}_L,\mathcal{N}_S)$ can be easily obtained by searching for the maximal value of $\left\|\mathbf{y}_M^{\{L+1\}}\right\| $ in $\mathcal{Y}_M$, e.g., testing throughout a finite number of vertices in polyhedral sets.

\section{Application to Assured Neural Network Compression}


\subsection{Assured Neural Network Compression}

In practical applications, neural networks are usually large in size, and it could be computationally expensive and time-consuming to perform those tasks requiring a large amount of computation resources. A promising method to mitigate the computation burden is to compress large-scale neural networks into small-scale ones and provide the approximate bisimulation error between two neural networks. With the approximate bisimulation error, we can infer the outputs of the original large-scale neural network via running its corresponding small-scale compressed one plus the approximate bisimulation error. The assured neural network compression is stated as below.
\begin{definition}
Given a large-scale neural network $\mathcal{N}_L$ with input set $\mathcal{X}$, a small-scale neural network $\mathcal{N}_S$ is called its assured compressed version with precision $\varepsilon$ if the approximate bisimulation error of two neural networks are not greater than  $\varepsilon$, i.e., 
\begin{align}
    d(\mathcal{N}_L,\mathcal{N}_S) \le \varepsilon 
\end{align}
where $\varepsilon \ge 0$.
\end{definition}
\begin{remark}
There exist a number of neural network compression methods \cite{deng2020model} to obtain small-scale neural network $\mathcal{N}_S$. In this paper, our focus is on how to compute the assured neural network compression precision $\varepsilon$ using the framework of approximate bisimulation relations proposed in the previous sections.  
\end{remark}
\begin{example}
We verify the effectiveness of the approximate bisimulation approach in neural network compression by a numerical case. In the numerical case, we aim to soundly simulate a neural network $\mathcal{N}_L$ (large-scale) with 5 hidden layers and 50 neurons in each hidden layer using a neural network $\mathcal{N}_S$ (small-scale) with 2 hidden layers and 10 neurons in each hidden layer. To facilitate the visualization of the simulation results, the output outputs of both neural networks are selected one-dimensional.

First, a neural network $\mathcal{N}_L$ is randomly generated, and then a neural network $\mathcal{N}_S$ is trained out of the input-output data of $\mathcal{N}_L$. All activation functions are ReLU functions. Using the merged neural network method and computing reachable set with NNV tool, the approximate bisimulation error $\varepsilon = 26.1227$ of the two neural networks can be obtained. With the help of $\varepsilon = 26.1227$, the upper and lower bounds of output $\mathbf{y}_L^{\{L\}}$ of $\mathcal{N}_L$ can be obtained via the outputs $\mathbf{y}_S^{\{S\}}$  of $\mathcal{N}_S$ with a smaller size, i.e., upper bound $\overline{\mathbf{y}}_L^{\{L\}}=\mathbf{y}_S^{\{S\}}+\varepsilon$ and lower bound $\underline{\mathbf{y}}_L^{\{L\}}=\mathbf{y}_S^{\{S\}}-\varepsilon$.

\begin{figure}[t!]
\centering
	\includegraphics[width=8.5cm]{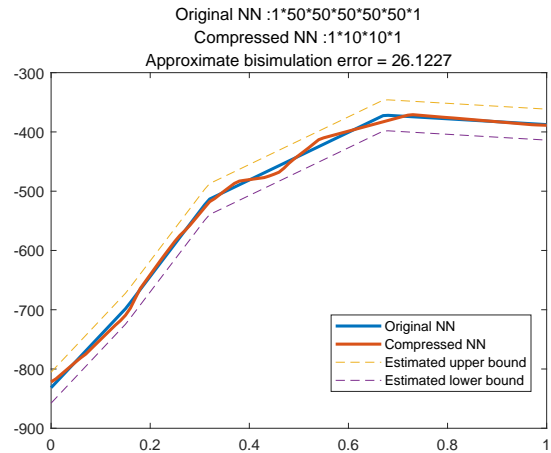}
	\caption{Assured compression for a random neural network (from $50 \times 50 \times 50 \times 50 \times 50$ to $10 \times 10$) by approximate bisimulation approach.}
	\label{numerical} 
\end{figure}

Output data of the original neural network and the compressed neural network, as well as the upper and lower bounds, are represented in Fig. \ref{numerical}. It can be observed that all the outputs $\mathbf{y}_L^{\{L\}}$ are within the upper bound $\overline{\mathbf{y}}_L^{\{L\}}$ and lower bound $\underline{\mathbf{y}}_L^{\{L\}}$, i.e., $\underline{\mathbf{y}}_L^{\{L\}}\le \mathbf{y}_L^{\{L\}} \le \overline{\mathbf{y}}_L^{\{L\}}$. 

\end{example}

\subsection{Application of ACAS Xu Network Verification}

\begin{figure}[b!]\label{acas_input_fig}
\centering
	\includegraphics[width=5cm]{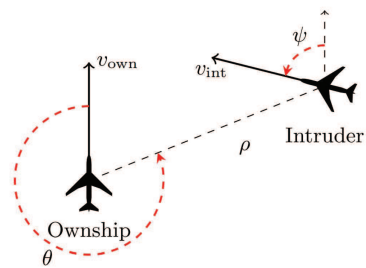}
	\caption{ACAS Xu horizontal logic table illustration \cite{katz2017reluplex}}
	\label{acas_input} 
\end{figure}

In this subsection, we apply the neural network model compression method to ACAS Xu network in \cite{owen2019acas} to accelerate the verification processes. ACAS Xu system has been developed using a large lookup table that maps sensor measurements to warning signals, see Fig. \ref{acas_input}. It has been shown that DNNs can significantly reduce memory (replacing a 2GB lookup table with an efficient DNNs of less than 3MB). 
The DNN method of ACAS Xu system consists of 45 DNNs, and each neural network contains 5 inputs and 5 outputs, with 6 hidden layers and 50 neurons with ReLU activation functions in each layer.

In practical applications, calculating the exact output reachable set of a neural network with 6 hidden layers and 50 neurons per layer requires huge computational effort and computational time \cite{katz2017reluplex}. Therefore, we compress the original neural networks into smaller neural networks and compute the assured precision by the approximate bisimulation method. Then, we can perform verification of properties based on those reduced-scale neural networks and approximate bisimulation error $\varepsilon$, i.e., expand the unsafe region $\neg \mathcal{S}$ in Definition \ref{def2}  by the approximate bismulation error $\varepsilon$.   

In this example, we use neural networks with two hidden layers and 10 neurons in each layer as the compressed version for the compression of the DNNs of the ACAS Xu system. Then, we verify Property $\phi_3$ on 27 neural networks in the ACAS Xu system using their assured compressed versions. The verification results and computational time are listed in Table \ref{ACAS_safety}. In Table \ref{ACAS_safety}, $\varepsilon$ is the approximate bisimulation error. $T_L$ is the verification time (seconds) using original neural networks and $T_S$ is the verification time (seconds) using compressed neural networks. $V_L$ is the verification results on original neural networks, and $V_S$ is the verification results on compressed neural networks.

As explicitly shown in Table \ref{ACAS_safety}, the verification time can be significantly reduced using compressed neural networks. It is worth mentioning that since the approximate bisimulation error is an over-approximation of the exact difference between the outputs of two neural networks, the safety conclusions based on compressed networks are only able to derive safe conclusions for original networks in safe cases. As to uncertain cases, we have to perform verification on original neural networks to ascertain the safety property. It can be found that the safety of 18 of the compressed neural networks can be used to conclude the safety of original neural networks. The remaining 9 unsafe verification results based on compressed neural networks are insufficient to derive safe or unsafe conclusions of original neural networks. This is mainly because the approximate bisimulation error is too large to meet the accuracy of the safety verification. Despite the 9 uncertain cases that need to be verified through original neural networks, the total verification time has been significantly reduced for these 27 neural networks.

\begin{table}[t!]
	\centering
	\caption{Property $\phi_3$ Verification for ACAS Xu System  }\label{ACAS_safety}
    \begin{tabular}{|c>{\columncolor[gray]{0.8}}cc>{\columncolor[gray]{0.8}}cc>{\columncolor[gray]{0.8}}c|}
    \hline 
ID  & $\varepsilon$ & $T_L (s)$ & $T_S (s)$ & $V_L$ & $V_S$ \\
 \hline\hline 
$N_{11}$ & 0.0927             & 463.24804        & 0.19383           & Safe      & \cellcolor{red!50} Uncertain     \\
$N_{12}$ & 0.089              & 504.08039        & 0.26257           & Safe      & \cellcolor{red!50} Uncertain     \\
$N_{13}$ & 0.0369             & 185.89549        & 0.66355           & Safe      & \cellcolor{red!50} Uncertain     \\
$N_{14}$ & 0.0041             & 29.31453         & 0.34665           & Safe      & Safe      \\
$N_{15}$ & 0.0026             & 45.7813          & 0.41446           & Safe      & Safe      \\
$N_{16}$ & 0.0013             & 12.17051         & 0.2766            & Safe      & Safe      \\
$N_{17}$ & 0.0018	          & 3.22305	         & 0.74309           & Unsafe     & \cellcolor{red!50} Uncertain     \\
$N_{18}$ & 0.0067	   	      & 2.53016	         & 0.50254           & Unsafe     & \cellcolor{red!50} Uncertain     \\
$N_{19}$ & 0.0056	          & 3.33068	         & 0.50024           & Unsafe     & \cellcolor{red!50} Uncertain     \\
$N_{21}$ & 0.1838             & 151.38468        & 1.2967            & Safe      & \cellcolor{red!50} Uncertain    \\
$N_{22}$ & 0.1143             & 56.81178         & 0.87974           & Safe      & \cellcolor{red!50} Uncertain     \\
$N_{23}$ & 0.018              & 92.08281         & 0.66704           & Safe      & Safe      \\
$N_{24}$ & 0.0035             & 3.14713          & 0.30876           & Safe      & Safe      \\
$N_{25}$ & 0.0031             & 19.24327         & 0.42653           & Safe      & Safe      \\
$N_{26}$ & 0.0161             & 2.77801          & 0.2835            & Safe      & Safe      \\
$N_{27}$ & 0.0047             & 9.83793          & 0.35039           & Safe      & Safe      \\
$N_{28}$ & 0.0063             & 2.87251          & 0.39635           & Safe      & Safe      \\
$N_{29}$ & 0.0022             & 1.51099          & 0.23274           & Safe      & Safe      \\
$N_{31}$ & 0.0244             & 63.11602         & 1.19615           & Safe      & Safe      \\
$N_{32}$ & 0.0907             & 421.81782        & 0.86584           & Safe      & \cellcolor{red!50} Uncertain     \\
$N_{33}$ & 0.0254             & 94.0685          & 0.19859           & Safe      & Safe      \\
$N_{34}$ & 0.0055             & 24.4508          & 0.38036           & Safe      & Safe      \\
$N_{35}$ & 0.002              & 8.88554          & 0.20696           & Safe      & Safe      \\
$N_{36}$ & 0.0135             & 18.18405         & 0.26895           & Safe      & Safe      \\
$N_{37}$ & 0.0136             & 1.25423          & 0.39768           & Safe      & Safe      \\
$N_{38}$ & 0.0061             & 5.36596          & 0.15807           & Safe      & Safe      \\
$N_{39}$ & 0.0055             & 11.92655         & 0.68403           & Safe      & Safe      \\
\hline 
	\end{tabular}
\end{table} 

 
\section{Conclusion}

This work proposed approximate bisimulation relations for feedforward neural networks. The approximate bisimulation relation formally define the maximal difference between the outputs of two bisimular neural networks from same inputs. An reachability-based computation procedure is developed to efficiently compute the approximation error via a novel neural network merging approach. Then, the approximation bismulation approach is applied to assured neural network compression. With the approximate bisimulation error, the perform tasks using the compressed network on behalf of original one such as verification of neural networks, which has been demonstrated by an ACAS Xu example.

\bibliographystyle{ieeetr}
\bibliography{ref}

\begin{thebibliography}{10}

\bibitem{Litjens2017}
G.~Litjens, T.~Kooi, B.~E. Bejnordi, A.~A.~A. Setio, F.~Ciompi, M.~Ghafoorian,
  J.~A. {van der Laak}, B.~{van Ginneken}, and C.~I. Sánchez, ``A survey on
  deep learning in medical image analysis,'' {\em Medical Image Analysis},
  vol.~42, pp.~60--88, 2017.

\bibitem{Schmidhuber2015}
J.~Schmidhuber, ``Deep learning in neural networks: An overview,'' {\em Neural
  Networks}, vol.~61, pp.~85--117, 2015.

\bibitem{Lawrence1997}
S.~Lawrence, C.~Giles, A.~C. Tsoi, and A.~Back, ``Face recognition: a
  convolutional neural-network approach,'' {\em IEEE Transactions on Neural
  Networks}, vol.~8, no.~1, pp.~98--113, 1997.

\bibitem{Hunt1992}
K.~Hunt, D.~Sbarbaro, R.~Żbikowski, and P.~Gawthrop, ``Neural networks for
  control systems—a survey,'' {\em Automatica}, vol.~28, no.~6,
  pp.~1083--1112, 1992.

\bibitem{Tong2016}
T.~Wang, H.~Gao, and J.~Qiu, ``A combined adaptive neural network and nonlinear
  model predictive control for multirate networked industrial process
  control,'' {\em IEEE Transactions on Neural Networks and Learning Systems},
  vol.~27, no.~2, pp.~416--425, 2017.

\bibitem{bojarski2016end}
M.~Bojarski, D.~Del~Testa, D.~Dworakowski, B.~Firner, B.~Flepp, P.~Goyal, L.~D.
  Jackel, M.~Monfort, U.~Muller, J.~Zhang, {\em et~al.}, ``End to end learning
  for self-driving cars,'' {\em arXiv preprint arXiv:1604.07316}, 2016.

\bibitem{Simon2020}
S.~Wiedemann, H.~Kirchhoffer, S.~Matlage, P.~Haase, A.~Marban, T.~Marinč,
  D.~Neumann, T.~Nguyen, H.~Schwarz, T.~Wiegand, D.~Marpe, and W.~Samek,
  ``Deepcabac: A universal compression algorithm for deep neural networks,''
  {\em IEEE Journal of Selected Topics in Signal Processing}, vol.~14, no.~4,
  pp.~700--714, 2020.

\bibitem{Yabo2019}
Y.~Zhang, W.~Ding, and C.~Liu, ``Summary of convolutional neural network
  compression technology,'' in {\em 2019 IEEE International Conference on
  Unmanned Systems (ICUS)}, pp.~480--483, 2019.

\bibitem{deng2020model}
L.~Deng, G.~Li, S.~Han, L.~Shi, and Y.~Xie, ``Model compression and hardware
  acceleration for neural networks: A comprehensive survey,'' {\em Proceedings
  of the IEEE}, vol.~108, no.~4, pp.~485--532, 2020.

\bibitem{szegedy2013intriguing}
C.~Szegedy, W.~Zaremba, I.~Sutskever, J.~Bruna, D.~Erhan, I.~Goodfellow, and
  R.~Fergus, ``Intriguing properties of neural networks,'' in {\em
  International Conference on Learning Representations}, 2014.

\bibitem{owen2019acas}
M.~P. Owen, A.~Panken, R.~Moss, L.~Alvarez, and C.~Leeper, ``Acas xu:
  Integrated collision avoidance and detect and avoid capability for uas,'' in
  {\em 2019 IEEE/AIAA 38th Digital Avionics Systems Conference (DASC)},
  pp.~1--10, IEEE, 2019.

\bibitem{tran2020nnv}
H.-D. Tran, X.~Yang, D.~M. Lopez, P.~Musau, L.~V. Nguyen, W.~Xiang, S.~Bak, and
  T.~T. Johnson, ``Nnv: The neural network verification tool for deep neural
  networks and learning-enabled cyber-physical systems,'' in {\em International
  Conference on Computer Aided Verification}, pp.~3--17, Springer, 2020.

\bibitem{xiang2018output}
W.~Xiang, H.-D. Tran, and T.~T. Johnson, ``Output reachable set estimation and
  verification for multilayer neural networks,'' {\em IEEE Transactions on
  Neural Networks and Learning Systems}, vol.~29, no.~11, pp.~5777--5783, 2018.

\bibitem{xiang2020reachable}
W.~Xiang, H.-D. Tran, X.~Yang, and T.~T. Johnson, ``Reachable set estimation
  for neural network control systems: A simulation-guided approach,'' {\em IEEE
  Transactions on Neural Networks and Learning Systems}, vol.~32, no.~5,
  pp.~1821--1830, 2021.

\bibitem{katz2017reluplex}
G.~Katz, C.~Barrett, D.~L. Dill, K.~Julian, and M.~J. Kochenderfer, ``Reluplex:
  An efficient smt solver for verifying deep neural networks,'' in {\em
  International Conference on Computer Aided Verification}, pp.~97--117,
  Springer, 2017.

\end{thebibliography}

\end{document}